\documentclass{article}

\usepackage{microtype}
\usepackage{graphicx}
\usepackage{subfigure}
\usepackage{booktabs} % for professional tables

\usepackage{hyperref}

\usepackage{thm-restate}

\usepackage[noend]{algorithmic}
\renewcommand{\algorithmicinput}{\textbf{Input: }}

\usepackage{amssymb}
\usepackage{amsthm}
\usepackage{amsmath}\allowdisplaybreaks

\usepackage{graphicx}
\usepackage{ifthen}

\newcommand{\E}{\mathbf{E}} % expectation

\newtheorem{definition}{Definition}
\newtheorem{lemma}{Lemma}
\newtheorem{theorem}{Theorem}
\newtheorem{assumption}{Assumption}

\newcommand{\compilefullversion}{true}%SHOW full version
\ifthenelse{\equal{\compilefullversion}{false}}{%
	\newcommand{\OnlyInFull}[1]{}
	\newcommand{\OnlyInShort}[1]{#1}
}{%
	\newcommand{\OnlyInFull}[1]{#1}%
	\newcommand{\OnlyInShort}[1]{}%
}%
\usepackage[accepted]{icml2020}

\icmltitlerunning{Optimization from Structured Samples for Coverage Functions}

\begin{document}

\twocolumn[
\icmltitle{Optimization from Structured Samples for Coverage Functions}

% It is OKAY to include author information, even for blind
% submissions: the style file will automatically remove it for you
% unless you've provided the [accepted] option to the icml2020
% package.

% List of affiliations: The first argument should be a (short)
% identifier you will use later to specify author affiliations
% Academic affiliations should list Department, University, City, Region, Country
% Industry affiliations should list Company, City, Region, Country

% You can specify symbols, otherwise they are numbered in order.
% Ideally, you should not use this facility. Affiliations will be numbered
% in order of appearance and this is the preferred way.
%\icmlsetsymbol{equal}{*}

\begin{icmlauthorlist}
	\icmlauthor{Wei Chen}{to}
	\icmlauthor{Xiaoming Sun}{goo,ed}
	\icmlauthor{Jialin Zhang}{goo,ed}
	\icmlauthor{Zhijie Zhang}{goo,ed}
\end{icmlauthorlist}

\icmlaffiliation{to}{Microsoft Research Asia, Beijing, China.}
\icmlaffiliation{goo}{CAS Key Lab of Network Data Science and Technology, Institute of Computing Technology, Chinese Academy of Sciences, Beijing, China.}
\icmlaffiliation{ed}{School of Computer Science and Technology, University of Chinese Academy of Sciences, Beijing, China.}

\icmlcorrespondingauthor{Wei Chen}{weic@microsoft.com}
\icmlcorrespondingauthor{Xiaoming Sun, Jialin Zhang, Zhijie Zhang}{\{sunxiaoming, zhangjialin, zhangzhijie\}@ict.ac.cn}

% You may provide any keywords that you
% find helpful for describing your paper; these are used to populate
% the "keywords" metadata in the PDF but will not be shown in the document
\icmlkeywords{Optimization from Samples, Maximum Coverage Problem, Negative Correlations}

\vskip 0.3in
]

% this must go after the closing bracket ] following \twocolumn[ ...

% This command actually creates the footnote in the first column
% listing the affiliations and the copyright notice.
% The command takes one argument, which is text to display at the start of the footnote.
% The \icmlEqualContribution command is standard text for equal contribution.
% Remove it (just {}) if you do not need this facility.

\printAffiliationsAndNotice{}  % leave blank if no need to mention equal contribution
%\printAffiliationsAndNotice{\icmlEqualContribution} % otherwise use the standard text.

\begin{abstract}
	We revisit the \emph{optimization from samples} (OPS) model, which studies the problem of optimizing objective functions
	directly from the sample data.
	Previous results showed that we cannot obtain a constant approximation ratio for the maximum coverage problem using polynomially many independent samples of the form $\{S_i, f(S_i)\}_{i=1}^t$ \cite{BalkanskiRS17}, even if coverage functions are $(1 - \epsilon)$-PMAC learnable using these samples \cite{BadanidiyuruDFKNR12}, which means most of the function values can be approximately learned very well with high probability.
	In this work, to circumvent the impossibility result of OPS, we propose a stronger model called \emph{optimization from structured samples} (OPSS) for coverage functions, where the data samples encode the structural information of the functions.
	We show that under three general assumptions on the sample distributions, we can design efficient OPSS algorithms that achieve a constant
	approximation for the maximum coverage problem.
	We further prove a constant lower bound under these assumptions, which is tight when not considering computational efficiency.
	Moreover, we also show that if we remove any one of the three assumptions, OPSS for the maximum coverage problem has no constant approximation.
\end{abstract}

\section{Introduction}
\label{Section: introduction}

Traditional optimization problems in the textbook are often formulated as mathematical models with specified parameters.
The computational task is to optimize an objective function given parameters of the model.
One such example is the maximum coverage problem.
Given a family of subsets $T_1, T_2, \cdots, T_n$ of a ground set $N$ and a positive integer $k$, the problem asks to find $k$ subsets whose union contains the most number of elements in $N$.
In practice, however, parameters of the model are often hidden in the complex real world and we cannot observe them directly.
Instead, we can only learn information about the model from the passively observed sample data.
Back to the maximum coverage problem, in this case we may not know the exact elements contained in every subset $T_i$, but only observe samples of subsets $T_i$'s, and for each sample we only
observe the number of elements it covers.
An immediate question, recently raised by Balkanski et al.~\yrcite{BalkanskiRS17}, asks to what extent we can optimize objective functions based on the sample data that we use to learn them.
More specifically, given samples $\{S_i,f(S_i)\}_{i=1}^t$ where $S_i$'s are drawn i.i.d.~from some distribution $\mathcal{D}$ on the subsets of
$N$, $f:2^N \rightarrow \mathbb{R}$ is an unknown objective function, and $t\in\mbox{poly}(|N|)$, can we solve $\max_{|S|\leq k} f(S)$?
For maximum coverage, $S_i$ would be a collection of some subsets $T_i$'s, function $f$ would be the number of elements covered by such collections.
Such problems form a new approach to optimization called \emph{optimization from samples} (OPS)~\cite{BalkanskiRS17}.

A reasonable and perhaps the most natural approach is to first learn a surrogate function $\tilde{f}:2^N \rightarrow \mathbb{R}$ which approximates well the original function $f$ and then optimize $\tilde{f}$ instead of $f$.
One may expect that if we can approximate a function well, then we can also optimize it well.
Standard frameworks of learnability in the literature include PAC learnability for boolean functions due to Valiant \yrcite{Valiant84} and PMAC learnability for real-valued set functions due to Balcan and Harvey \yrcite{BalcanH11}.

Unfortunately, the learning-and-then-optimization approach does not work in general.
Indeed, Balkanski et al.~\yrcite{BalkanskiRS17} show the striking result that the maximum coverage problem cannot be approximated within a ratio better than $2^{-\Omega(\sqrt{\log |N|})}$ using only polynomially many samples drawn i.i.d.~from \emph{any} distribution,
even though (a) for any constant $\epsilon > 0$, coverage functions are $(1-\epsilon)$-PMAC learnable over \emph{any} distribution
\cite{BadanidiyuruDFKNR12}, which means most of the function values can be approximately learned very well with high probability; and (b) maximum coverage problem as a special case of submodular function maximization has a
$1-1/e$ approximation given a value oracle to the coverage function~\cite{NemhauserWF78}.

The impossibility result by Balkanski et al.~\yrcite{BalkanskiRS17} uses coverage functions defined over a partition of the ground set,
which ensure the ``good'' and ``bad'' parts of the partition cannot be distinguished from the samples.
In other words, the impossibility result arises because the samples do not provide information on the structure of coverage functions.

To circumvent the above impossibility result,
we propose a stronger model called \emph{optimization from structured samples} (OPSS) for coverage functions, which encodes structural information of the coverage functions into the samples.
In many real-world applications, such structural information are often revealed in the data, for example, a crowd-sourcing platform records
the crowd-workers' coverage on the tasks they took, a document analysis application records the keywords coverage on the documents they appear,
etc.
Thus the OPSS model is reasonable in practice.
However, even in the stronger OPSS model, not all sample distributions will allow a constant approximation for the maximum coverage problem.
In this paper, we study the assumptions that enable constant approximation in the OPSS model and its related algorithmic and hardness results.
We now state our model and results in more detail.

\subsection{Model}
\label{Subsection: model}

For sake of comparison, we first state the definition of optimization from samples \cite{BalkanskiRS17} for general set functions.

\begin{definition}[Optimization from samples (OPS)]
	Let $\mathcal{F}$ be a class of set functions defined on the ground set $L$.
	$\mathcal{F}$ is  $\alpha$-optimizable from samples in constraint $\mathcal{M} \subseteq 2^L$ over distribution $\mathcal{D}$ on $2^L$,
	if there exists a (not necessarily polynomial time) algorithm such that,
	given any parameter $\delta > 0$ and sufficiently large $L$, there exists some integer $t_0 \in \mbox{poly}(|L|, 1 / \delta)$,
	for all $t \geq t_0$, for any set of samples $\{S_i, f(S_i)\}_{i=1}^t$ with $f \in \mathcal{F}$ and $S_i$'s drawn
	i.i.d.~from $\mathcal{D}$, the algorithm takes samples $\{S_i, f(S_i)\}_{i=1}^t$ as the input and
	returns $S\in \mathcal{M}$ such that
	\[ \Pr_{S_1,\cdots,S_t\sim\mathcal{D}}[\E[f(S)]\geq \alpha\cdot\max_{T\in\mathcal{M}}f(T)] \geq 1 - \delta, \]
	where the expectation is taken over the randomness of the algorithm.
\end{definition}

Next we state the definition of coverage functions in terms of bipartite graphs as well as the definition of optimization from structured samples for coverage functions.

\begin{definition}[Coverage functions]
	Assume there is a bipartite graph $G=(L, R, E)$.
	For node $u \in L\cup R$, let $N_G(u)$ denote its neighbors in $G$.
	The neighbors of a subset $S \subseteq L$ or $S \subseteq R$ is $N_G(S) = \cup_{u\in S} N_G(u)$.
	The coverage function $f_G:2^L \rightarrow \mathbb{R}_+$ is the number of neighbors covered by a set $S \subseteq L$, i.e.~$f_G(S) = |N_G(S)|$.
\end{definition}

\begin{definition}[Optimization from structured samples (OPSS)]
	Let $\mathcal{F}$ be the class of coverage functions defined on all bipartite graphs $\{G=(L,R,E)\}$ with two components $L$ and $R$.
	$\mathcal{F}$ is  $\alpha$-optimizable under OPSS in constraint $\mathcal{M} \subseteq 2^L$ over distribution $\mathcal{D}$ on $2^L$,
	if there exists a (not necessarily polynomial time) algorithm such that,
	given any parameter $\delta > 0$ and sufficiently large $L$,
	there exists some integer $t_0 \in \mbox{poly}(|L|, |R|, 1 / \delta)$,
	for all $t \geq t_0$, for any set of samples  $\{S_i, N_G(S_i)\}_{i=1}^t$ with $f_G \in \mathcal{F}$ and $S_i$'s drawn
	i.i.d.~from $\mathcal{D}$, the algorithm takes samples  $\{S_i, N_G(S_i)\}_{i=1}^t$ as the input and
	returns $S\in \mathcal{M}$ such that
	\[ \Pr_{S_1,\cdots,S_t\sim\mathcal{D}}[\E[f_G(S)]\geq \alpha\cdot\max_{T\in\mathcal{M}}f_G(T)] \geq 1 - \delta, \]
	where the expectation is taken over the randomness of the algorithm.
\end{definition}

Samples in OPSS are {\em structured} in that the exact members covered by a set $S\subseteq L$
are revealed, instead of only the number of covered members being revealed as in OPS.
In this paper we focus on the cardinality constraint $\mathcal{M}_{\le k} = \{S \subseteq L \mid |S| \leq k\}$.
Maximizing coverage functions under this constraint is known as the \emph{maximum coverage problem}.

Our OPSS model is defined so far only for coverage functions.
One reason is that the impossibility of OPS given by Balkanski et al.~\yrcite{BalkanskiRS17} is on the
coverage functions, which is striking because coverage functions admit a simple constant approximation
algorithm with the value oracle and is $(1-\epsilon)$-PMAC learnable as mentioned before.
Thus coverage function is the first to consider for circumventing the impossibility result for OPS.
Another reason is that coverage functions exhibit natural structures via the bipartite graph representation.
Other set functions may exhibit different combinatorial structures and thus the OPSS problem may need to be defined
accordingly to reflect the specific structural information for other set functions.

\subsection{Our Results}
\label{Subsection: our results}

One of our main results is to provide a set of three general assumptions on the sample distribution together with an algorithm and show
that the algorithm achieves a constant approximation ratio for the maximum coverage problem in OPSS under the assumption.
The general assumption is summarized below.

\begin{assumption}
	\label{assump:dist}
	We assume that the distribution $\mathcal{D}$ on $2^L$ satisfy the following three assumptions:
	\vspace{-4mm}
	\begin{enumerate}
		\setlength{\itemsep}{-1mm}
		\item[1.1] \textbf{Feasibility.} A sample $S \sim \mathcal{D}$ is always feasible, i.e.~$|S| \leq k$.
		\item[1.2] \textbf{Polynomial bounded sample complexity.} For any $u \in L$, the probability $p_u = \Pr_{S \sim \mathcal{D}}[u \in S]$ satisfies $p_u \geq 1 / |L|^c$ for some constant $c$.
		\item[1.3] \textbf{Negative correlation.} The random variables $X_u = \mathbf{1}_{u \in S}$ are ``negatively correlated'' (see Definition \ref{Definition: negative correlation}) over distribution $\mathcal{D}$.
	\end{enumerate}	
\end{assumption}

All three assumptions above are natural.
In particular, Assumption \ref{assump:dist}.2 means that all elements in the ground set have sufficient probability to be sampled, and Assumption \ref{assump:dist}.3 means
informally that
the appearance of one element in the sampled set $S$ would reduce the probability of the appearance of another element in $S$.
In fact, typical distributions over $\mathcal{M}_{\le k}$, such as uniform distribution $\mathcal{D}_{\le k}$ over all subsets in
$\mathcal{M}_{\le k}$ or uniform distribution $\mathcal{D}_k$ over all subsets of exact size $k$, all satisfy these assumptions.
Our result based on the above assumption is summarized by the following theorem.

\begin{theorem}
	\label{Theorem: optimizable distributions}
	If a distribution $\mathcal{D}$ satisfies Assumption~\ref{assump:dist}, given any
	$\alpha$-approximation algorithm $A$ for the standard maximum coverage problem, coverage functions are
	$\frac{\alpha}{2}$-optimizable under OPSS in the cardinality constraint $\mathcal{M}_{\le k}$ over $\mathcal{D}$
	for any $k\le |L|$.
	Furthermore, the OPSS algorithm uses a polynomial number of arithmetic operations and one call of algorithm~$A$.
\end{theorem}

The general approximation ratio $\alpha$ is to cover both polynomial-time and non-polynomial-time algorithms.
If we need a polynomial-time algorithm, then we know that the best ratio we can achieve is $1-1/e$ if NP$\ne$P~\cite{NemhauserWF78,Feige98}.
Thus our OPSS algorithm achieves $\frac{1}{2} (1-1/e)$ approximation.
If running time is not our concern, then we can use $\alpha=1$ by an exhaustive search algorithm, and
our OPSS algorithm achieves $\frac{1}{2}$ approximation.

We further show that if the distribution is $\mathcal{D}_k$, i.e.~the uniform distribution over all subsets of exact size $k$,
we have another OPSS algorithm to achieve $ (\alpha - \epsilon)$ approximation, as shown below.
This implies that our OPSS algorithm (almost) matches the approximation ratio of any algorithm for the standard maximum coverage problem.
\begin{theorem}
	\label{Theorem: bound for uniform distribution}
	For any constant $\epsilon > 0$, given any $\alpha$-approximation algorithm $A$ for the standard maximum coverage problem,
	coverage functions are $(\alpha - \epsilon)$-optimizable under OPSS  in the cardinality constraint $\mathcal{M}_{\le k}$ over $\mathcal{D}_k$, assuming that $\ln^2 |L| \leq k \leq |L|/2$ and $|R| \leq \frac{\epsilon}{2}|L|^{(\epsilon\ln |L|)/8}$.
	Furthermore, the OPSS algorithm uses a polynomial number of arithmetic operations and one call of algorithm $A$.
\end{theorem}

Next, we prove a hardness result showing that the approximation ratio of $\frac{1}{2}$ is unavoidable for some distributions, which means that
when efficiency is not the concern, our upper and lower bounds are tight.
\begin{restatable}{theorem}{thmlowerbound}
	\label{Theorem: 1 / 2 hardness}
	There is a distribution $\mathcal{D}$ satisfying Assumption~\ref{assump:dist} such that coverage functions
	are not $\alpha$-optimizable under OPSS in the cardinality constraint $\mathcal{M}_{\le k}$ over $\mathcal{D}$ for
	any $\alpha > \frac{1}{2} + o(1)$.
\end{restatable}

Finally, we also show that the three conditions given in Assumption \ref{assump:dist} are necessary, in the sense that dropping any one of them would result in no constant approximation for the OPSS problem.
This demonstrates that our three conditions need to work together to make OPSS solvable.

\begin{theorem}
	\label{Theorem: conditions are necessary}
	By dropping any one of the conditions in Assumption \ref{assump:dist}, there is a distribution $\mathcal{D}$ such that coverage functions
	are not $\alpha$-optimizable under OPSS for any constant $\alpha$ in the cardinality constraint $\mathcal{M}_{\le k}$ over $\mathcal{D}$.
\end{theorem}

To summarize, in this paper we investigate the structural information on coverage functions that could allow us to circumvent the impossibility result
in \cite{BalkanskiRS17}.
We show that when the samples could reveal the covered elements rather than just the count, under certain reasonable assumptions on
the sample distribution (Assumption~\ref{assump:dist}), we could design an OPSS algorithm that achieves $\alpha/2$ approximation, where $\alpha$ is the approximation ratio
of a standard maximum coverage problem.
Moreover, for the uniform distribution on subsets of size $k$, we provide an efficient algorithm that achieves tight $\alpha - \epsilon$ approximation, matching
the performance of any algorithm for the standard maximum coverage problem.
On the lower bound side, we show that the approximation ratio of $1/2$ is unavoidable, which matches the upper bound when not considering computational complexity.
Finally, we show that removing any one of the three conditions in Assumption~\ref{assump:dist}, we cannot achieve constant approximation for OPSS.
Our study opens up the possibility of studying structural information for achieving optimization from samples, which is needed in many applications in the big data era.

\subsection{Related Work}
\label{Subsection: related work}

The study of optimization from samples (OPS) was initiated by Balkanski et al.~\yrcite{BalkanskiRS17}.
They proved that no algorithm can achieve an approximation ratio better than $2^{-\Omega(\sqrt{\log n})}$ for the maximum coverage problem under OPS.
The same set of authors showed there is an optimal $(1-c)/(1+c-c^2)$ approximation algorithm for maximizing monotone submodular functions with curvature $c$ subject to a cardinality constraint over uniform distributions under OPS \cite{BalkanskiRS16}.
For submodular function minimization, it was proved in \cite{BalkanskiS17} that no algorithm can obtain an approximation strictly better than $2 - o(1)$ under OPS. And this is tight via a trivial $2$-approximation algorithm.
Rosenfeld et al.~\yrcite{RosenfeldBGS18} defined a weaker variant of OPS called \emph{distributionally optimization from samples} (DOPS).
They showed that a class of set functions is optimizable under DOPS if and only if it is PMAC-learnable.

%\OnlyInFull{Due to the space constraint, some proofs are included in the appendix.}\OnlyInShort{Due to the space constraint,
%	some proofs are included in the supplementary material.}

\section{Concepts and Tools}
\label{Section: concepts and tools}

We first discuss the definition of negative correlation.
Negative dependence among random variables has been extensively studied in the literature and there are a lot of qualitative versions of this concept \cite{JogdeoP75, karlinR80, Ghosh81, BlockSS82, JoagP83}.
Among them, the most widely accepted one is the \emph{negative association} (NA) defined in \cite{JoagP83}. However, in this paper, we only use a weaker version of NA. Thus, more distributions satisfy our definition of negative correlation. It is also easy to see that the uniform distributions $\mathcal{D}_k$ and $\mathcal{D}_{\leq k}$ both satisfy this definition.

\begin{definition}[Negative correlation]
	\label{Definition: negative correlation}
	A set of $0$-$1$ random variables $X_1, \cdots, X_n$ is negative correlated, if for any disjoint subsets $I, J \subseteq [n]:=\{1,\cdots,n\}$,
	\[ \E\big[\prod_{i\in I\cup J}(1-X_i)\big] \leq \E\big[\prod_{i\in I}(1-X_i)\big]\E\big[\prod_{j\in J}(1-X_j)\big]. \]
\end{definition}

Then we prove the following lemma, which shows that the occurrence of an event would reduce the probability of occurrences of other events.

\begin{restatable}{lemma}{lemnegativecor}
	\label{Lemma: negative correlation}
	Assume that $X_1, \cdots, X_n$ are negatively correlated $0$-$1$ random variables.
	Then for any $I \subseteq [n]$ and $j \notin I$,
	\[ \Pr[ \vee_{i\in I} (X_i = 1) \mid X_j = 1] \leq \Pr[\vee_{i\in I} (X_i = 1)]. \]
\end{restatable}
\begin{proof}
	Since $X_1, \cdots, X_n$ are negatively correlated,
	\[ \Pr[\wedge_{i\in I \cup \{j\}}(X_i = 0)] \leq \Pr[\wedge_{i\in I}(X_i = 0)]\Pr[X_j = 0], \]
	which is equivalent to
	\begin{align*}
	&   \Pr[\wedge_{i\in I}(X_i = 0)] - \Pr[\wedge_{i\in I}(X_i = 0), X_j =1] \\
	&\leq \Pr[\wedge_{i\in I}(X_i = 0)]\Pr[X_j = 0].
	\end{align*}
	Rearranging the last inequality, we have
	\begin{align*}
	& \Pr[\wedge_{i\in I}(X_i = 0)]\Pr[X_j = 1] \\
	&\leq \Pr[\wedge_{i\in I}(X_i = 0), X_j =1],
	\end{align*}
	which is equivalent to
	\begin{align*}
	&   (1 - \Pr[\vee_{i \in I} (X_i = 1)])\Pr[X_j = 1] \\
	&\leq \Pr[X_j = 1] - \Pr[\vee_{i \in I} (X_i = 1), X_j =1].
	\end{align*}
	Rearranging the last inequality, we have
	\begin{align*}
	& \Pr[\vee_{i\in I} (X_i = 1), X_j = 1] \\
	& \leq \Pr[\vee_{i\in I} (X_i = 1)] \Pr[X_j = 1].
	\end{align*}
	This concludes the proof.
\end{proof}

Next is Chernoff bound used in the analysis of probability concentration.

\begin{lemma}[Chernoff bound, \cite{MitzenmacherU05}] \label{lem:chernoff}
	Let $X_1, X_2, \cdots, X_n$ be independent random variables in $\{0, 1\}$ with $\Pr[X_i = 1] \geq p_i$.
	Let $X = \sum_{i=1}^{n} X_i$ and $\E[X] = \mu \geq \mu_L = \sum_{i=1}^{n}p_i$.
	Then, for $0 < \delta < 1$,
	\[ \Pr[X \leq (1 - \delta) \mu_L] \leq e^{-\mu_L\delta^2/2}. \]
\end{lemma}

\section{Constant Approximations for OPSS}
\label{Section: a constant approximation algorithm}
In this section, we present two constant approximation algorithms for OPSS and their results:
one for the general distributions satisfying Assumption~\ref{assump:dist} (Theorem~\ref{Theorem: optimizable distributions})
and the other for the uniform distribution $\mathcal{D}_k$ (Theorem~\ref{Theorem: bound for uniform distribution}).

\subsection{A Constant Approximation under Assumption~\ref{assump:dist}}

\begin{algorithm}[t]
	\caption{OPSS algorithm for the general Assumption~\ref{assump:dist}}
	\label{Algorithm: main algorithm}
	\begin{algorithmic}[1]
		\INPUT{Samples $\{S_i, N_G(S_i)\}_{i=1}^t$ and $k \in \mathbb{N}_+$}
		\STATE Let $T_1 = S_1$
		\STATE Construct a surrogate bipartite graph $\tilde{G} = (L, R, \tilde{E})$ such that
		%(a) $\tilde{L} = \cup_{i=1}^t S_i$ and $\tilde{R} = \cup_{i=1}^t N_G(S_i)$, and (b)
		for each $u \in L$, $N_{\tilde{G}}(u) = \cap_{i: u \in S_i} N_G(S_i)$
		
		\STATE Let $T_2 = A(\tilde{G}, k) $
		
		\STATE \textbf{return} $T_1$ with probability $1/2$; and $T_2$ otherwise
	\end{algorithmic}
\end{algorithm}

The algorithm is shown in Algorithm \ref{Algorithm: main algorithm}.
It returns one of the two solutions $T_1$ and $T_2$ with equal probability,
where $T_1$ is just the first sample, and $T_2$ is the solution of an $\alpha$-approximation algorithm $A$ on
a constructed surrogate bipartite graph $\tilde{G}$ for the standard maximum coverage problem.
The parameters of algorithm $A$ denote the graph and the constraint respectively.
The surrogate graph $\tilde{G}=(L,R,\tilde{E})$ is constructed from samples $\{S_i, N_G(S_i)\}_{i=1}^t$  such that
for each node $u\in L$, we construct $u$'s coverage in $R$ as $N_{\tilde{G}}(u) = \cap_{i: u \in S_i} N_G(S_i)$,
which is an estimate of $N_G(u)$.
The intuition is as follows.
If some singleton $\{u\}$ is drawn from $\mathcal{D}$, the knowledge about $N_G(u)$ is completely revealed.
However, it might be the case that $\mathcal{D}$ always returns a large set $S$, and the exact knowledge about $N_G(u)$ for $u \in S$ is hidden behind $N_G(S)$.
Thus to reveal as much knowledge about $N_G(u)$ as possible, it is natural to use the intersection of samples that contain $u$ as an estimate.

The difficulty in the analysis is that $N_{\tilde{G}}(u)$ is always an overestimate of $N_G(u)$,
and it is impossible to show that $N_{\tilde{G}}(u)$ is a good approximation of $N_G(u)$.
One extreme example is that suppose for some $v\in L$, $\Pr_{S\sim \mathcal{D}}[v\in S]=1$, then we have that
$N_{\tilde{G}}(u)$ always contains all elements in $N_G(u)\cup N_G(v)$, which might be much larger than $N_G(u)$ itself. Thus $T_2$ itself might not be a good solution on the original graph $G$.
To circumvent this difficulty, the key step is to show that for any $S \sim \mathcal{D}$,
$N_{\tilde{G}}(T_2) \backslash N_G(T_2) \subseteq \cup_{u \in L} (N_{\tilde{G}}(u) \backslash N_G(u)) \subseteq N_G(S)$ with high probability (Lemma \ref{Lemma: lower bound of a sample}). Consequently, $N_{\tilde{G}}(T_2) \subseteq N_G(T_1 \cup T_2)$ and we can obtain a constant approximation ratio by combining a random sample $T_1$ with $T_2$ as in Algorithm~\ref{Algorithm: main algorithm}.
Note that $T_1$ and $T_2$ may be correlated since they are both dependent on $S_1$, but this is not an issue based on our analysis.

\begin{lemma}
	\label{Lemma: lower bound of a sample}
	For a given $\delta > 0$, suppose that the number of samples $t \geq \frac{4|L|^c|R|}{\delta} \ln \frac{4|L||R|}{\delta}$, where $c$ is the constant
	in Assumption~\ref{assump:dist}.2.
	Under Assumption~\ref{assump:dist}, we have
	\[ \Pr_{S_1,\cdots,S_t\sim\mathcal{D}}[\cup_{u \in L} (N_{\tilde{G}}(u) \backslash N_G(u)) \subseteq N_G(S_1)] \geq 1 - \delta. \]
\end{lemma}

The proof of Lemma \ref{Lemma: lower bound of a sample} is delayed to Section \ref{Subsection: proof of theorem lower bound of S_1}.
For now, we use it to prove Theorem \ref{Theorem: guarantee of main algorithm}, which is a more concrete version of Theorem \ref{Theorem: optimizable distributions}.

\begin{theorem}
	\label{Theorem: guarantee of main algorithm}
	If a distribution $\mathcal{D}$ satisfies Assumption~\ref{assump:dist}, given any $\alpha$-approximation algorithm $A$ for the standard maximum coverage problem,
	coverage functions are $\frac{\alpha}{2}$-optimizable under OPSS in the cardinality constraint $\mathcal{M}_{\le k}$ over $\mathcal{D}$ for any $k \le |L|$.
	More precisely, for any $\delta > 0$, suppose that the number of samples $t \geq \frac{4|L|^c|R|}{\delta} \ln \frac{4|L||R|}{\delta}$, where $c$ is the constant
	in Assumption~\ref{assump:dist}.2.
	Let $ALG$ be the solution returned by Algorithm \ref{Algorithm: main algorithm} and $OPT$ be the optimal solution on the original graph $G$.
	Then under Assumption~\ref{assump:dist}, we have
	\[ \Pr_{S_1,\cdots,S_t\sim\mathcal{D}} \left[\E[f_G(ALG)] \geq \frac{\alpha}{2} f_G(OPT) \right] \geq 1 - \delta. \]
\end{theorem}

\begin{proof}
	By the construction of $\tilde{G}$, $N_G(u) \subseteq N_{\tilde{G}}(u)$ for any $u \in L$.
	Therefore, $G$ is a subgraph of $\tilde{G}$ and $f_{\tilde{G}}(OPT) \geq f_G(OPT)$.
	Since $A$ is an $\alpha$ approximation algorithm,
	\[ f_{\tilde{G}}(T_2) \geq \alpha f_{\tilde{G}}(OPT) \geq \alpha f_{G}(OPT). \]
	On the other hand, it holds that $N_{\tilde{G}}(T_2) \backslash N_G(T_2) \subseteq \cup_{u \in T_2} (N_{\tilde{G}}(u) \backslash N_G(u)) \subseteq \cup_{u \in L} (N_{\tilde{G}}(u) \backslash N_G(u))$.
	Since $T_1 = S_1$, by Lemma \ref{Lemma: lower bound of a sample}, it holds with probability $1 - \delta$ that $N_{\tilde{G}}(T_2) \backslash N_G(T_2) \subseteq N_G(T_1)$, and
	\begin{align*}
	f_{\tilde{G}}(T_2) &= |N_G(T_2) \cup (N_{\tilde{G}}(T_2) \backslash N_G(T_2))| \\
	&\leq |N_G(T_2)| + |N_{\tilde{G}}(T_2) \backslash N_G(T_2)| \\
	&\leq |N_G(T_2)| + |N_G(T_1)| \\
	&= f_G(T_2) + f_G(T_1).
	\end{align*}
	Therefore, with probability $1 - \delta$,
	\begin{align*}
	\E[f_G(ALG)] & = \E \left[\frac{1}{2} \cdot f_G(T_1) + \frac{1}{2} \cdot  f_G(T_2) \right] \\
	& \geq \E\left[\frac{1}{2}f_{\tilde{G}}(T_2) \right]\geq \frac{\alpha}{2} f_G(OPT).
	\end{align*}
\end{proof}

For common distributions, the constant $c$ in Assumption \ref{assump:dist}.2 is usually small, thus Algorithm \ref{Algorithm: main algorithm} requires moderately small number of samples.
For instance, for distributions $\mathcal{D}_k$ and $\mathcal{D}_{\leq k}$, $\Pr_{S \sim \mathcal{D}_k}[u \in S] = k/|L|$ and $\Pr_{S \sim \mathcal{D}_{\leq k}}[u \in S] \geq 1/|L|$.
Thus both distributions require only $O(\frac{|L||R|}{\delta} \ln \frac{|L||R|}{\delta})$ samples.

\subsubsection{Proof of Lemma \ref{Lemma: lower bound of a sample}}
\label{Subsection: proof of theorem lower bound of S_1}

We first introduce some notations.
Let $|L| = n$, $|R| = m$ and $\overline{t} = \frac{2m}{\delta} \ln \frac{4mn}{\delta}$.
For any node $u \in L$, let $t_u = |\{i: u \in S_i\}|$ be the number of samples where $u$ appears.
For any node $v \in R$, let $q_v = \Pr_{S \sim \mathcal{D}}[v \in N_G(S)] $ be the probability that $v$ is covered by a sample $S \sim \mathcal{D}$.
Our analysis starts with partitioning $R$ into two subsets $R_1$ and $R_2$, where $R_1 = \{v \in R \mid q_v \leq 1 - \frac{\delta}{2m}\}$ and $R_2 = R \backslash R_1$.
In general, we will show that nodes in $R_1$ will not appear in $\cup_{u \in L} (N_{\tilde{G}}(u) \backslash N_G(u))$ with high probability (Lemma \ref{Lemma: R_1 is not error}) and $R_2$ will be covered by any sample $S \sim \mathcal{D}$  with high probability (Lemma \ref{Lemma: R_2 will be covered}).
These facts together suffice to prove Lemma \ref{Lemma: lower bound of a sample}.

\begin{lemma}
	\label{Lemma: aux-1}
	Assume that $t \geq 2n^c \cdot \overline{t}$.
	For fixed $u \in L$, $\Pr_{S_1,\cdots,S_t\sim\mathcal{D}}[t_u \leq \overline{t}] \leq \delta / (4mn)$.
\end{lemma}

\begin{proof}
	For fixed $u \in L$, let $X_i = 1$ if $u \in S_i$ and $0$ otherwise.
	Then $t_u = \sum_{i=1}^{t} X_i$.
	By Assumption~\ref{assump:dist}.2, $p_u = \Pr_{S \sim \mathcal{D}}[u \in S] \geq 1 / n^c$. Thus $\E[t_u] \geq t / n^c \geq 2\overline{t}$.
	By Chernoff bound (Lemma~\ref{lem:chernoff}),
	\[ \Pr[t_u \leq \overline{t}] = \Pr\left[t_u \leq \left(1 - \frac{1}{2} \right) \cdot 2\overline{t} \right] \leq e^{-\overline{t} / 4} \leq \frac{\delta}{4mn}. \]
	The last inequality needs $m\ge 2\delta$, which is satisfied for all nontrivial instances.
\end{proof}

\begin{lemma}
	\label{Lemma: consequence of negative correlation}
	For any $u \in L$ and $v \in R$ such that $(u, v) \notin E$,
	$\Pr_{S \sim \mathcal{D}}[v \in N_G(S), u \in S] \leq \Pr_{S \sim \mathcal{D}}[v \in N_G(S)]\Pr_{S \sim \mathcal{D}}[u \in S]$.
\end{lemma}

\begin{proof}
	Just note that the event $\{v \in N_G(S)\}$ is equivalent to $\{\cup_{u' \in N_G(v)} (u' \in S) \}$.
	The lemma follows directly from Lemma \ref{Lemma: negative correlation}.
\end{proof}

\begin{lemma}
	\label{Lemma: aux-2}
	For any $u \in L$ and $v \in R$ such that $(u, v) \not \in E$, $\Pr_{S_1,\cdots,S_t\sim\mathcal{D}}[v \in N_{\tilde{G}}(u) \backslash N_G(u), t_u = \ell] \leq q_v^l \cdot \Pr_{S_1,\cdots,S_t\sim\mathcal{D}}[t_u = \ell]$,  for any $\ell \in \mathbb{N}$.
\end{lemma}

\begin{proof}
	
	By the law of total probability, the formula on the left-hand side is equal to $\sum_{I\subseteq [t]:|I|=\ell}
	\Pr[v \in N_{\tilde{G}}(u) \backslash N_G(u), u \in \cap_{i\in I} S_i,  u \notin \cup_{j \notin I} S_j]$.
	Since $S_i$'s are independent samples, by construction of $N_{\tilde{G}}(u)$ and Lemma \ref{Lemma: consequence of negative correlation}, we have
	\begin{align*}
	&\Pr[v \in N_{\tilde{G}}(u) \backslash N_G(u), u \in \cap_{i\in I} S_i,  u \notin \cup_{j \notin I} S_j] \\
	&= \Pr[v \in \cap_{i\in I} N_G(S_i), u \in \cap_{i\in I} S_i, u \notin \cup_{j \notin I} S_j] \\
	&= \prod_{i \in I}\Pr[v \in N_G(S_i), u \in S_i]\prod_{j \notin I}\Pr[u \notin S_j] \\
	&\leq \prod_{i \in I}\left(\Pr[v \in N_G(S_i)] \Pr[u \in S_i]\right)\prod_{j \notin I}\Pr[u \notin S_j] \\
	&= \prod_{i \in I}\Pr[v \in N_G(S_i)] \prod_{i \in I} \Pr[u \in S_i]\prod_{j \notin I}\Pr[u \notin S_j] \\
	&= q_v^{\ell} \cdot \Pr[u \in \cap_{i\in I} S_i, u \notin \cup_{j \notin I} S_j].
	\end{align*}
	Thus
	\begin{align*}
	& \Pr[v \in N_{\tilde{G}}(u) \backslash N_G(u), t_u = \ell] \\
	& \leq q_v^{\ell} \sum_{I\subseteq [t]:|I|=\ell} \Pr[u \in \cap_{i\in I} S_i, u \notin \cup_{j \notin I} S_j] \\
	& = q_v^{\ell} \cdot \Pr[t_u = \ell].
	\end{align*}
\end{proof}

\begin{lemma}
	\label{Lemma: R_1 is not error}
	Assume that $t \geq 2n^c \overline{t}$. Then
	$\Pr_{S_1,\cdots,S_t\sim\mathcal{D}}[R_1 \cap (\cup_{u \in L} (N_{\tilde{G}}(u) \backslash N_G(u))) = \emptyset] \geq 1 - \delta / 2$.
\end{lemma}

\begin{proof}
	For node $v \in R_1$ and node $u \in L$ such that $(u,v) \notin E$,
	we have
	\begin{align*}
	& \Pr[v \in N_{\tilde{G}}(u) \backslash N_G(u)] \\
	&= \sum_{\ell \geq 0} \Pr[v \in N_{\tilde{G}}(u) \backslash N_G(u), t_u = \ell] \\
	&\leq \sum_{\ell \geq 0} \Pr[t_u = \ell] \cdot q_v^{\ell} \\
	&\leq \sum_{\ell \leq \overline{t}} \Pr[t_u = \ell] \cdot 1 + \sum_{\ell > \overline{t}} \Pr[t_u = \ell] \cdot q_v^{\overline{t}} \\
	&= \Pr[t_u \leq \overline{t}] + \Pr[t_u > \overline{t}] \cdot q_v^{\overline{t}} \\
	&\leq \frac{\delta}{4mn} + \left(1-\frac{\delta}{2m}\right)^{\frac{2m}{\delta}\ln{\frac{4mn}{\delta}}} \\
	&\leq \frac{\delta}{4mn} + \frac{\delta}{4mn}=\frac{\delta}{2mn}.
	\end{align*}
	The first inequality holds due to Lemma \ref{Lemma: aux-2}.
	The second to last inequality holds due to Lemma \ref{Lemma: aux-1}, the fact that $q_v \leq 1 - \frac{\delta}{2m}$ for all $v\in R_1$
	and $\overline{t} = \frac{2m}{\delta} \ln \frac{4mn}{\delta}$.
	Finally, by union bound, we have
	\begin{align*}
	& \Pr[R_1 \cap (\cup_{u \in L} (N_{\tilde{G}}(u) \backslash N_G(u))) \neq \emptyset] \\
	&= \Pr[\exists\, v \in R_1, u \in L \mbox{ s.t.~} v \in N_{\tilde{G}}(u) \backslash N_G(u) ]  \\
	&\leq \sum_{v \in R_1, u \in L} \Pr[v \in N_{\tilde{G}}(u) \backslash N_G(u) ] \\
	&\leq \sum_{v \in R_1, u \in L} \frac{\delta}{2mn} \leq \frac{\delta}{2}.
	\end{align*}
	The proof is completed.
\end{proof}

\begin{lemma}
	\label{Lemma: R_2 will be covered}
	$\Pr_{S_1 \sim \mathcal{D}}[R_2 \subseteq N_G(S_1)] \geq 1 - \delta / 2$.
\end{lemma}

\begin{proof}
	For a node $v \in R_2$, by definition, $\Pr_{S \sim \mathcal{D}}[v \notin N_G(S)] = 1 - q_v \leq \frac{\delta}{2m}$.
	By union bound, we have $\Pr_{S \sim \mathcal{D}}[\exists\, v \in R_2 \mbox{ s.t.~} v \notin N_G(S)] \leq \delta / 2$.
	That is, $\Pr_{S \sim \mathcal{D}}[R_2 \subseteq N_G(S)] \geq 1 - \delta / 2$.
\end{proof}

\emph{Proof of Lemma \ref{Lemma: lower bound of a sample}.}
By Lemma \ref{Lemma: R_1 is not error}, with probability $1 - \delta / 2$, $R_1 \cap (\cup_{u \in L} (N_{\tilde{G}}(u) \backslash N_G(u))) = \emptyset$ and therefore $\cup_{u \in L} (N_{\tilde{G}}(u) \backslash N_G(u)) \subseteq R_2$.
On the other hand, by Lemma \ref{Lemma: R_2 will be covered}, with probability $1 - \delta / 2$, $R_2 \subseteq N_G(S_1)$.
Finally, by union bound, $\cup_{u \in L} (N_{\tilde{G}}(u) \backslash N_G(u)) \subseteq N_G(S_1)$ with probability $1 - \delta$. \qed

\subsection{A Tight Algorithm for OPSS under $\mathcal{D}_k$}

In this section, we present a tight algorithm for OPSS under distribution $\mathcal{D}_k$, the uniform distribution over all subsets of size $k$.
Compared with Algorithm \ref{Algorithm: main algorithm}, Algorithm \ref{Algorithm: tight algorithm under D_k} takes an additional input $\epsilon\in(0,1)$ and has two other modifications.
First, when constructing $T_2$, the constraint is replaced by $|S| \leq (1 - \epsilon / 2)k$, which only incurs little loss in the approximation ratio.
Second, instead of assigning a sample $S \sim \mathcal{D}_k$ to $T_1$, the algorithm picks a set uniformly at random from all subsets of size $\epsilon k / 2$ and assigns it to $T_1$.
The key observation is that under distribution $\mathcal{D}_k$, although $T_1$ is quite small, it suffices to cover nodes in $N_{\tilde{G}}(T_2) \backslash N_G(T_2)$ with high probability. However, this is not true for general distributions.
As a result, $T_1 \cup T_2$ yields an $\alpha - \epsilon$ approximation for the problem, and it is also feasible.

\begin{algorithm}[t]
	\caption{Tight OPSS algorithm under $\mathcal{D}_k$}
	\label{Algorithm: tight algorithm under D_k}
	\begin{algorithmic}[1]
		\INPUT Samples $\{S_i, N_G(S_i)\}_{i=1}^t$, $k \in \mathbb{N}_+$, $\epsilon \in (0, 1)$
		\STATE Draw a set $T_1$ from $\mathcal{D}_{\epsilon k/2}$.
		\STATE Construct a surrogate bipartite graph $\tilde{G} = (L, R, \tilde{E})$ such that
		%		(a) $\tilde{L} = \cup_{i=1}^t S_i$ and $\tilde{R} = \cup_{i=1}^t N_G(S_i)$, and (b)
		for each $u \in L$, $N_{\tilde{G}}(u) = \cap_{S_i: u \in S_i} N_G(S_i)$
		\STATE Let $T_2 = A(\tilde{G}, (1 - \epsilon/2)k)$
		\STATE \textbf{return} $T_1 \cup T_2$
	\end{algorithmic}
\end{algorithm}

We begin the analysis with some notations.
Let $|L| = n$, $|R| = m$ and $\overline{t} = \left(\frac{2m}{\epsilon}\right)^{\frac{8}{\epsilon}}\ln \frac{2mn}{\delta}$.
In the analysis, we assume that $\ln^2 n \leq k \leq n/2$ and $m \leq \frac{\epsilon}{2}n^{(\epsilon\ln n)/8}$.
This is a sufficient condition for a key inequality, as we will further explain after Theorem~\ref{Theorem: tight bound under D_k}.
For any node $u \in L$, let $t_u = |\{i: u \in S_i\}|$ be the number of samples where $u$ appears.
For any node $v \in R$, let $q_v = \Pr_{S \sim \mathcal{D}_k}[v \in N_G(S)] $ be the probability that $v$ is covered by a sample $S \sim \mathcal{D}_k$.
Let $d(v) = |N(v)|$ denote the number of $v$'s neighbors.
Partition $R$ into two subsets $R_1$ and $R_2$, where $R_1 = \{v \in R \mid d(v)  < \frac{2n}{\epsilon k} \ln \frac{2m}{\epsilon}\}$ and $R_2 = R \backslash R_1$.
While in the general case discussed in previous section, $R$ is partitioned according to the value of $q_v$, here we partition $R$ according to the value of $d(v)$.
The reason is that $\mathcal{D}_k$ is a uniform distribution.
Thus for $v \in R$, the more neighbors it has, the higher probability it will be covered by a sample $S \sim \mathcal{D}_k$.
The observation is further formulated as Lemma \ref{Lemma: d(v) and q_v}.
Based on it, we can show that with high probability nodes in $R_1$ will not appear in $\cup_{u \in L} (N_{\tilde{G}}(u) \backslash N_G(u))$ (Lemma \ref{Lemma: R_1 is not error, D_k}).
Besides, $q_v$ increases exponentially with respect to $d(v)$.
Thus instead of picking a sample from $\mathcal{D}_k$, drawing a set $T_1$ from $\mathcal{D}_{\epsilon k /2}$ suffices to cover nodes in $R_2$ (Lemma \ref{Lemma: R_2 covered by T_1}).
%Let $R_1 = \{v \in R \mid d(v) k < \frac{\gamma+1}{\epsilon} n \log n\}$ and $R_2 = R \backslash R_1$.

\begin{lemma}
	\label{Lemma: d(v) and q_v}
	For any $v \in R_1$, $q_v \leq 1 - \left(\frac{\epsilon}{2m}\right)^{8 / \epsilon}$.
\end{lemma}

\begin{proof}
	It is easy to verify that when $\ln^2 n \leq k$ and $m \leq \frac{\epsilon}{2}n^{(\epsilon\ln n)/8}$,
	we have
	$\frac{2n}{\epsilon k} \ln \frac{2m}{\epsilon} \leq n/4$.
	%\end{equation}
	Thus for any $v \in R_1$, $d(v)  < \frac{2n}{\epsilon k} \ln \frac{2m}{\epsilon} \leq  n/4$.
	Together with  $k\le n/2$, we have
	\begin{align*}
	1 - q_v &= \frac{{n - d(v) \choose k}}{{n \choose k}} \\
	&= \frac{(n - d(v)) \cdots (n - d(v) - k + 1)}{n \cdots (n - k + 1)} \\
	&\geq \left(1 - \frac{d(v)}{n - k + 1}\right)^k \\
	&\geq \left(1 - \frac{d(v)}{n/2}\right)^k \\%\ \ \ \ (\mbox{since }k\leq n/2)\\
	%   	&\geq& \left(1 - \frac{d(v)}{n - k + 1}\right)^k \\
	%		&\geq& \exp\left(- \frac{2kd(v)}{n - k + 1}\right) \\
	&\geq \exp\left(- \frac{4kd(v)}{n}\right) \\
	%		&\geq& \exp\left(-\frac{(4n / \epsilon) \ln (2m / \epsilon)}{n/2+1}\right) \\
	&\geq \exp\left(-(8 / \epsilon) \ln (2m / \epsilon)\right) \\
	&= (\epsilon/2m)^{8 / \epsilon}.
	\end{align*}
	The third inequality holds since $1 - x \geq e^{-2x}$ for $x \in [0,1/2]$.
	The last inequality holds since %$k \leq n / 2$ and
	$d(v) < \frac{2n}{\epsilon k} \ln \frac{2m}{\epsilon}$ for $v \in R_1$.
\end{proof}

Similar to Lemmas \ref{Lemma: R_1 is not error} and~\ref{Lemma: R_2 will be covered}, we show the following lemmas.
The proofs are included in Section \ref{Subsection: detailed proofs under D_k}.
%\OnlyInShort{The proofs are included in the supplementary material.}\OnlyInFull{The proofs are included in Appendix~\ref{app:omittedproofs}.}

\begin{restatable}{lemma}{lemuniformRone}
	\label{Lemma: R_1 is not error, D_k}
	Assume that $t \geq 2(n/k) \cdot \overline{t}$. We have
	\[ \Pr_{S_1,\cdots,S_t\sim\mathcal{D}_k}[R_1 \cap (\cup_{u \in L} (N_{\tilde{G}}(u) \backslash N_G(u))) = \emptyset] \geq 1 - \delta. \]
\end{restatable}

\begin{restatable}{lemma}{lemuniformRtwo}
	\label{Lemma: R_2 covered by T_1}
	$\Pr_{T_1 \sim \mathcal{D}_{\epsilon k/2}}[R_2 \subseteq N_G(T_1)] \geq 1 - \epsilon / 2$.
\end{restatable}

Now we prove Theorem \ref{Theorem: tight bound under D_k}, which is a more concrete version of Theorem \ref{Theorem: bound for uniform distribution}.

\begin{theorem}
	\label{Theorem: tight bound under D_k}
	For any constant $\epsilon > 0$, given any $\alpha$-approximation algorithm $A$ for the standard maximum coverage problem,
	coverage functions are $(\alpha - \epsilon)$-optimizable under OPSS in the cardinality constraint $\mathcal{M}_{\leq k}$ over $\mathcal{D}_k$,
	assuming that $\ln^2 |L| \leq k \leq |L|/2$ and $|R| \leq \frac{\epsilon}{2}|L|^{(\epsilon\ln |L|)/8}$.
	More precisely, for any $\delta > 0$, suppose that the number of samples $t \geq \frac{2|L|}{k} \left(\frac{2|R|}{\epsilon}\right)^{\frac{8}{\epsilon}}\ln \frac{2|L||R|}{\delta}$.
	Let $ALG$ be the solution returned by Algorithm \ref{Algorithm: tight algorithm under D_k} and $OPT$ be the optimal solution on the original graph $G$.
	Then
	\[ \Pr_{S_1,\cdots,S_t\sim\mathcal{D}_k}[\E[f_G(ALG)] \geq (\alpha-\epsilon) f_G(OPT)] \geq 1 - \delta. \]
\end{theorem}

\begin{proof}
	By the construction of $\tilde{G}$, $N_G(u) \subseteq N_{\tilde{G}}(u)$ for any $u \in L$.
	Therefore, $G$ is a subgraph of $\tilde{G}$ and $f_{\tilde{G}}(OPT) \geq f_G(OPT)$.
	Let $OPT_k$ be the optimal solution when selecting $k$ elements.
	Since $A$ is an $\alpha$ approximation algorithm and $|T_2| \leq (1 - \epsilon / 2) k$,
	\begin{align*}
	& f_{\tilde{G}}(T_2) \geq \alpha f_{\tilde{G}}(OPT_{(1 - \epsilon / 2) k}) \\
	& \ge \alpha(1-\epsilon/2) f_{\tilde{G}}(OPT_k)  \geq \alpha(1-\epsilon/2) f_{G}(OPT),
	\end{align*}
	where the second inequality above utilizes the submodularity of the coverage functions.
	
	Let $\mathcal{E}$ be the event $R_1 \cap (\cup_{u \in L} N_{\tilde{G}}(u) \backslash N_G(u)) = \emptyset$.
	By Lemma \ref{Lemma: R_1 is not error, D_k}, $\Pr_{S_1,\cdots,S_t\sim\mathcal{D}_k}[\mathcal{E}] \geq 1 - \delta $.
	
	We now assume that event $\mathcal{E}$ holds.
	In this case, we first have $N_{\tilde{G}}(T_2) \backslash N_G(T_2) \subseteq \cup_{u \in L} N_{\tilde{G}}(u) \backslash N_G(u) \subseteq R_2$.
	Next, conditioned on $\mathcal{E}$, we still have the claim in Lemma \ref{Lemma: R_2 covered by T_1} because the sampling of $T_1$ is independent of
	the sampling of $S_1, \ldots, S_t$.
	Therefore, when $\mathcal{E}$ holds, we have
	\begin{align*}
	& \E[f_G(ALG)] = \E[f_G(T_1 \cup T_2)] \\
	&\geq \Pr[R_2 \subseteq N_G(T_1)] \E[f_G(T_1 \cup T_2) \mid R_2 \subseteq N_G(T_1)] \\
	%    	&= \Pr[R_2 \subseteq N_G(T_1)] \E[|N_G(T_1) \cup N_G(T_2)| \mid R_2 \subseteq N_G(T_1)] \\
	&\ge \Pr[R_2 \subseteq N_G(T_1)] \E[|N_{\tilde{G}}(T_2)| \mid R_2 \subseteq N_G(T_1)] \\
	%    	&= \Pr[R_2 \subseteq N_G(T_1)] \E[f_{\tilde{G}}(T_2) \mid R_2 \subseteq N_G(T_1)] \\
	&\geq  \alpha (1-\epsilon/2)^2 f_G(OPT) \geq (\alpha - \epsilon) f_G(OPT).
	\end{align*}
	This concludes the proof.
\end{proof}

We remark that Lemma \ref{Lemma: d(v) and q_v} (and thus Theorem \ref{Theorem: tight bound under D_k}) holds as long as $k \leq |L|/2$ and $|R| \leq \frac{\epsilon}{2}e^{(\epsilon k)/8}$.
The technical condition $\ln^2 |L| \leq k \leq |L|/2$ and $|R| \leq \frac{\epsilon}{2}|L|^{(\epsilon\ln |L|)/8}$
is indeed a relaxed sufficient condition by setting a lower bound for $k$.
However, it provides a reasonable asymptotic requirement on $k$ and $|R|$ in terms of $|L|$.
%We remark that the technical condition  $\ln^2 |L| \leq k \leq |L|/2$ and $|R| \leq \frac{\epsilon}{2}|L|^{(\epsilon\ln |L|)/8}$
%is a relaxed sufficient condition for Inequality~\eqref{eq:keyinequality}.
%It provides a reasonable asymptotic requirement on $k$ and $|R|$ in terms of $|L|$, but technically as long as Inequality~\eqref{eq:keyinequality}
%holds, the condition is not necessary.

\subsubsection{Proof of Lemmas \ref{Lemma: R_1 is not error, D_k} and \ref{Lemma: R_2 covered by T_1}}
\label{Subsection: detailed proofs under D_k}

\begin{lemma}
	\label{lemma: consequence of chernoff bound}
	Assume that $t \geq 2(n/k) \cdot \overline{t}$, where $\overline{t} = \left(\frac{2m}{\epsilon}\right)^{\frac{8}{\epsilon}}\ln \frac{2mn}{\delta}$.
	For fixed $u \in L$, $\Pr[t_u \leq \overline{t}] \leq \delta / (2mn)$.
\end{lemma}

\begin{proof}
	For fixed $u \in L$, let $X_i = 1$ if $u \in S_i$ and $0$ otherwise.
	Then $t_u = \sum_{i=1}^{t} X_i$.
	Since $p_u = \Pr_{S \sim \mathcal{D}_k}[u \in S] \geq k / n$, $\E[t_u] \geq tk / n = 2\overline{t}$.
	By Chernoff bound (Lemma \ref{lem:chernoff}),
	\[ \Pr[t_u \leq \overline{t}] = \Pr\left[t_u \leq \left(1 - \frac{1}{2} \right) \cdot 2\overline{t} \right] \leq e^{-\overline{t} / 4} \leq \frac{\delta}{2mn}. \]
	The last inequality holds as long as $\overline{t} \geq 4 \ln \frac{2mn}{\delta}$.
\end{proof}

\emph{Proof of Lemma \ref{Lemma: R_1 is not error, D_k}.}
	For node $v \in R_1$ and node $u \in L$ such that $(u, v) \not\in E$,
	we have
	\begin{align*}
	&  \Pr[v \in N_{\tilde{G}}(u) \backslash N_G(u)] \\
	&= \sum_{\ell \geq 0} \Pr[v \in N_{\tilde{G}}(u) \backslash N_G(u), t_u = \ell] \\
	&\leq \sum_{\ell \geq 0} \Pr[t_u = \ell] \cdot q_v^{\ell} \\
	&\leq \sum_{\ell \leq \overline{t}} \Pr[t_u = \ell] \cdot 1 + \sum_{\ell > \overline{t}} \Pr[t_u = \ell] \cdot q_v^{\overline{t}} \\
	&= \Pr[t_u \leq \overline{t}] + \Pr[t_u > \overline{t}] \cdot q_v^{\overline{t}} \\
	&\leq \frac{\delta}{2mn} + \left(1 - \left(\frac{\epsilon}{2m}\right)^{\frac{8}{\epsilon}}\right)^{ \left(\frac{2m}{\epsilon}\right)^{\frac{8}{\epsilon}}\ln \frac{2mn}{\delta}} \\
	&\leq \frac{\delta}{2mn} + \frac{\delta}{2mn} = \frac{\delta}{mn}.
	\end{align*}
	The first inequality holds due to Lemma \ref{Lemma: aux-2}.
	The second to last inequality holds due to Lemma \ref{lemma: consequence of chernoff bound} and Lemma \ref{Lemma: d(v) and q_v}, and the fact that $\overline{t} = \left(\frac{2m}{\epsilon}\right)^{\frac{8}{\epsilon}}\ln \frac{2mn}{\delta}$.
	
	Finally, by union bound, we have
	\begin{align*}
	&  \Pr[R_1 \cap (\cup_{u \in L} N_{\tilde{G}}(u) \backslash N_G(u)) \neq \emptyset] \\
	&= \Pr[\exists\, v \in R_1, u \in L \mbox{ s.t.~} v \in N_{\tilde{G}}(u) \backslash N_G(u) ]  \\
	&\leq \sum_{v \in R_1, u \in L} \Pr[v \in N_{\tilde{G}}(u) \backslash N_G(u) ] \\
	&\leq \sum_{v \in R_1, u \in L} \delta / (mn) \leq \delta.
	\end{align*}
	The proof is completed. \qed

\emph{Proof of Lemma \ref{Lemma: R_2 covered by T_1}.}
	For node $v \in R_2$, $d(v) \geq \frac{2n}{\epsilon k} \ln \frac{2m}{\epsilon}$, then
	\begin{align*}
	& \Pr_{T_1 \sim \mathcal{D}_{\epsilon k/2}}[v \not\in N_G(T_1)] = \frac{{n-d(v) \choose \epsilon k/2}}{{n \choose \epsilon k/2}} \\
	& = \frac{(n-d(v))\cdots(n-d(v) - \epsilon k /2 + 1)}{n \cdots (n - \epsilon k /2 + 1)} \\
	&\leq \left(1 - \frac{d(v)}{n}\right)^{\epsilon k/2} \leq \exp\left(-\frac{\epsilon kd(v)}{2n}\right) \\
	&\leq \frac{\epsilon}{2m}.
	\end{align*}
	%	\[ \Pr_{A_1 \sim \mathcal{D}_{\epsilon k/2}}[v \not\in N_G(A_1)] = \frac{{n-d(v) \choose \epsilon k/2}}{{n \choose \epsilon k/2}} \leq \left(1 - \frac{d(v)}{n}\right)^{\epsilon k/2} \leq \frac{\epsilon}{2m}. \]
	By union bound, we have
	\[ \Pr_{T_1 \sim \mathcal{D}_{\epsilon k/2}}[\exists\, v \in R_2, v \not\in N_G(T_1)] \leq \epsilon / 2. \]
	That is, $\Pr_{T_1 \sim \mathcal{D}_{\epsilon k/2}}[R_2 \subseteq N_G(T_1)] \geq 1 - \epsilon / 2$. \qed

\section{Hardness Results for OPSS}
\label{Section: all the conditions are necessary}

\subsection{The $1/2$ Hardness for OPSS under Assumption \ref{assump:dist}}

{\thmlowerbound*}
\begin{proof}
	The distribution $\mathcal{D}$ is constructed as follows.
	Number nodes in $L$ such that $L = \{u_1, \cdots, u_n\}$.
	Let $L_1$ contain the first $k-1$ nodes and $L_2 = L \backslash L_1$.
	Any sample $S$ from $\mathcal{D}$ always contains the $k - 1$ nodes in $L_1$.
	The last node in $S$ is picked uniformly at random from $L_2$.
	It is easy to see that distribution $\mathcal{D}$ satisfies Assumption \ref{assump:dist}.
	
	Next, we construct a class of graphs $G_1, \cdots, G_{k-1}$ as follows such that they cannot be distinguished from the samples.
	(a) For any $i \leq k - 1$ and $u, v \in L$, $N_{G_i}(u) \cap N_{G_i}(v) = \emptyset$;
	(b) for any $i, j \leq k - 1$ and $u \in L_2$, $|N_{G_i}(u)| = r$ and $N_{G_i}(u) = N_{G_j}(u)$;
	(c) for any $i \leq k - 1$ and $u \in L_1$ with $u \neq u_i$, $N_{G_i}(u) = \emptyset$;
	(d) for any $i \le k-1$, and $u_i$ covers the same set of $(k-1)r$ nodes across different graph $G_i$'s.
	Clearly, the optimal solution $OPT_i$ of $G_i$ contains node $u_i$ and arbitrary $k-1$ nodes in $L_2$.
	Thus $f_{G_i}(OPT_i) = 2(k-1)r$.
	
	We prove the desired ratio by a probabilistic argument.
	Let $B$ be any (randomized) OPSS algorithm and $T$ be the solution it returns.
	Let $G$ be a graph drawn uniformly at random from $G_1, \cdots, G_{k-1}$.
	Since any sample of $\mathcal{D}$ always return the first $k-1$ nodes, and the union coverage of these $k-1$ nodes
	is always the same across different graphs $G_i$'s,
	solution $T$ is independent of the random choice of $G$,
	although it may be dependent on the random choices in the samples from nodes in $L_2$.
	Suppose the solution $T$ of $B$ is fixed.
	Let $x = |T \cap L_1|, 0 \leq x \leq k - 1$.
	By the above argument that $T$ and $G$ are independent, we have that
	the expected number of nodes covered by $T$ is
	\begin{eqnarray*}
		& & \E_G[f_G(T)\mid \mbox{solution $T$ of $B$ are fixed}] \\
		&=& \frac{x}{k-1}(k-1)r + (k - x)r = kr.
	\end{eqnarray*}
	%	Since $\mathcal{B}$ cannot distinguish those graphs from the samples, the choices of $\mathcal{B}$ are independent of the choice of $G$.
	As a result, $\E_{B,G}[f_G(T)] = kr$, which implies there must be a $G$ from $G_1, \cdots, G_{k-1}$ such that $\E_{B}[f_G(T)] \leq kr$.
	Thus $\E_{B}[f_G(T)] / f_G(OPT) \leq k/(2(k-1)) = 1 / 2 + o(1)$.
\end{proof}

\subsection{Assumption \ref{assump:dist} Is Necessary}

In this section, we show that the three conditions in Assumption \ref{assump:dist} are necessary, in the sense that dropping any one of them would result in no constant approximation for the OPSS problem.
The necessity of Assumption \ref{assump:dist}.2 is relatively trivial, and we include it in \OnlyInShort{the full version.}\OnlyInFull{Appendix~\ref{app:omittedproofs}.}
%The necessity of Assumption \ref{assump:dist}.2 is relatively trivial, and we include it in \OnlyInShort{the supplementary material.}\OnlyInFull{Appendix~\ref{app:omittedproofs}.}

\paragraph{Assumption \ref{assump:dist}.1 Is Necessary.}
For the distribution which always returns $L$, no reasonable algorithm exists for the OPSS problem.
Thus it is easy to see that we cannot drop Assumption \ref{assump:dist}.1 without any restriction.
Instead, we show that even if we relax assumption \ref{assump:dist}.1 a little bit, no constant approximation algorithm exists.

\begin{theorem}
	\label{Theorem: assumption 1 is necessary}
	Let $\mathcal{D}_r$ be the uniform distribution over all subsets of size $r = \omega(k \log^2 |L|)$.
	The coverage functions are not $\alpha$-optimizable under OPSS for any constant $\alpha$ in the cardinality constraint $\mathcal{M}_{\leq k}$ over $\mathcal{D}_r$.
\end{theorem}

\begin{proof}
	Clearly, $\mathcal{D}_r$ satisfies Assumption \ref{assump:dist}.2 and \ref{assump:dist}.3, but not Assumption \ref{assump:dist}.1.
	Let $|L| = n$, $|R| = m = \mbox{poly}(n)$, and $p = r / \log^2 n = \omega(k)$.
	We first construct a class of graphs $G_1, \cdots, G_p$ where $G_i = (L, R, E_i)$.
	Let $L$ be partitioned into disjoint subsets $\{L_1, \cdots, L_p\}$, each with $q = n / p$ nodes.
	For graph $G_i$, $L_i$ is \emph{good} in that for any $u \in L_i$, $N_{G_i}(u) = R$;
	each $L_j$ with $j \neq i$ is \emph{bad} in that $N_{G_i}(L_j) = \emptyset$.
	Clearly, the optimal solution of any graph covers $m$ nodes.
	
	Next we show that with high probability $G_1, \cdots, G_p$ cannot be distinguished from the samples.
	For $L_i$ of $G_i$,
	\[ \Pr_{S \sim \mathcal{D}_r}[S \cap L_i = \emptyset ] = \frac{{n - q \choose r}}{{n \choose r}} \leq \left(1 - \frac{q}{n}\right)^r \leq e^{-\log^2 n}. \]
	Thus for $t = \mbox{poly}(|L|, |R|) = \mbox{poly}(n)$ samples, by the union bound,
	\[ \Pr_{S_1,\cdots,S_t\sim\mathcal{D}_r}[\exists\, j \in [t] \mbox{ s.t. } S_j \cap L_i = \emptyset] \leq te^{-\log^2 n} = o(1). \]
	Hence with probability $1 - o(1)$, $N_{G_i}(S_j) = R$ for all $S_j$ and all $G_i$.
	Below we assume this is exactly the case for all $G_i$'s, which means no algorithm can distinguish these $G_i$'s from the samples.
	
	We prove the desired ratio by a probabilistic argument.
	Let $B$ be any (randomized) algorithm and $T$ be the solution it returns.
	Let $G$ be a graph drawn uniformly at random from $G_1, \cdots, G_p$ and $L_g$ be the good part of $G$.
	Suppose the solution $T$ of $B$ is fixed.
	Since $|T| \leq k$, it can touch at most $k$ $L_j$'s.
	Thus $\Pr_G[T \cap L_g \neq \emptyset \mid \mbox{the solution $T$ of } B \mbox{ is fixed}] \leq k / p = o(1)$.
	Since $B$ cannot distinguish those graphs from the samples, the solution $T$ of $B$ is independent of the random graph $G$.
	As a result, $\Pr_{B, G}[T \cap L_g \neq \emptyset] = o(1)$ and $\E_{B, G}[f_G(T)] = o(1) \cdot m$, which implies there must be a $G$ from $G_1, \cdots, G_p$ such that $\E_{B}[f_G(T)] = o(1) \cdot m$.
	Thus $\E_{B}[f_G(T)] / f_G(OPT) = o(1)$ with probability $1 - o(1)$.
\end{proof}

As a complement of Theorem \ref{Theorem: assumption 1 is necessary}, we show that as long as $r=O(k)$, we have a constant approximation algorithm for the OPSS problem.
\OnlyInFull{The proof is included in Appendix~\ref{app:omittedproofs}.}\OnlyInShort{The proof is included in the full version.}
%\OnlyInFull{The proof is included in Appendix~\ref{app:omittedproofs}.}\OnlyInShort{The proof is included in the supplementary material.}
\begin{restatable}{theorem}{thmOkconstant} \label{thm:Okconstant}
	Let $\mathcal{D}_r$ be the uniform distribution over all subsets of size $r = O(k)$.
	The coverage functions are $\alpha$-optimizable under OPSS for some constant $\alpha$ in the cardinality constraint $\mathcal{M}_{\leq k}$ over $\mathcal{D}_r$.
\end{restatable}

\paragraph{Assumption \ref{assump:dist}.3 Is Necessary.}
Assumption \ref{assump:dist}.3 plays a central role in the analysis of our algorithms.
Thus it is reasonable to consider its necessity.
In this section we show that this is exactly the case.

\begin{theorem}
	There is a distribution $\mathcal{D}$, which satisfies Assumption \ref{assump:dist}.1 and \ref{assump:dist}.2, but not Assumption \ref{assump:dist}.3,
	such that coverage functions are not $\alpha$-optimizable under OPSS for any constant $\alpha$ in the cardinality constraint $\mathcal{M}_{\leq k}$ over $\mathcal{D}$.
\end{theorem}

\begin{proof}
	The distribution $\mathcal{D}$ is constructed as follows.
	Let $L$ be partitioned into $n / k$ disjoint subsets $L_1, \cdots, L_{n/k}$; each $L_j$ contains exactly $k$ nodes.
	A sample $S \sim \mathcal{D}$ is drawn uniformly at random from $L_1, \cdots, L_{n/k}$.
	Clearly, this distribution satisfies Assumption \ref{assump:dist}.1 and \ref{assump:dist}.2, but it is not negatively correlated.
	
	Let $G$ be a random graph constructed with the following properties:
	(a) $N_G(L_i) \cap N_G(L_j) = \emptyset$ for any $i \neq j$;
	(b) $|N_G(L_i)| = r$ for all $i \leq n / k$;
	(c) within each $L_i$, there is a node $u^i$ such that $N_G(u^i) = N_G(L_i)$;
	(d) for node $u \in L_i$ with $u \neq u^i$, $N_G(u) = \emptyset$;
	(e) node $u^i$ is determined by selecting a uniformly random node from $L_i$.
	All the possible outcomes of $G$ form the graph class $\mathcal{G}$.
	It is easy to see graphs from $\mathcal{G}$ cannot be distinguished from the samples.
	The optimal solution of any graph from $\mathcal{G}$ covers $kr$ nodes.
	
	Now we prove the desired ratio by a probabilistic argument.
	Let $B$ be any (randomized) algorithm and $T$ be the solution it returns.
	Suppose the solution $T$ of $B$ is fixed. Then
	\begin{align*}
	&  \E_G[f_G(T) \mid \mbox{the solution $T$ of $B$ is fixed}] \\
	&= \sum_{j=1}^{n/k} \Pr_G[T \mbox{ contains } u^j \mbox{ of } L_j] \cdot r
	= \sum_{j=1}^{n/k} \frac{|T \cap L_j|}{k}\cdot r = r.
	\end{align*}
	Since $B$ cannot distinguish those graphs from the samples, the solution $T$ of $B$ is independent of the random graph $G$.
	As a result, $\E_{B,G}[f_G(T)] = r$, which implies there must be
	some fixed $G$ in the graph class $\mathcal{G}$ such that $\E_{B}[f_G(T)] \leq r$.
	Thus $\E_{B}[f_G(T)] / f_G(OPT) \leq 1 / k$.
\end{proof}

\section{Future Work}

One immediate question is to close the $[\frac{1}{2}(1-e^{-1}), \frac{1}{2}]$ gap of polynomial time algorithms under Assumption \ref{assump:dist} in our model.
Besides, it is interesting to define suitable structured samples for other set functions and investigate the possibility of optimization for those functions.
One concrete example of such functions is the probabilistic coverage function where each edge $(u,v)$ in the bipartite graph
$G=(L,R,E)$ has a probability indicating the probability that $u$ covers $v$.

% Acknowledgements should only appear in the accepted version.
\section*{Acknowledgements}

This work was supported in part by the National Natural Science Foundation of China Grants No.~61832003, 61761136014, 61872334,  the 973 Program of China Grant No.~2016YFB1000201, K.C.~Wong Education Foundation.

% In the unusual situation where you want a paper to appear in the
% references without citing it in the main text, use \nocite
%\nocite{langley00}

\bibliography{OPSS}
\bibliographystyle{icml2020}

\clearpage

\OnlyInFull{
	\appendix
	\section*{Appendix}
	\section{Omitted Proofs}
	\label{app:omittedproofs}

We first prove Theorem~\ref{thm:Okconstant}.
We utilize the following lemma,, which is easy to prove for coverage functions or can be derived from Lemma 2.2 in Feige et al.~\yrcite{FeigeMV11}.
\begin{lemma}
	\label{Lemma: property of submodular functions}
	Let $f: 2^L \rightarrow \mathbb{R}$ be a coverage function. For any set $T\subseteq L$, let $T(p)$ be a random subset of $T$ where each element appears with probability at least $p$ (not necessarily independently).
	Then $\mathbf{E}[f(T(p))] \geq p \cdot f(T)$.
\end{lemma}

{\thmOkconstant*}
\begin{proof}
	The algorithm is as follows.
	It first invokes Algorithm \ref{Algorithm: main algorithm} to obtain a solution $T_1$ with $|T_1| \leq r$.
	Then let $T_2$ be a uniformly random subset of $T_1$ with size $k$.
	The algorithm returns $T_2$ as a solution.
	
	By Theorem \ref{Theorem: guarantee of main algorithm}, $T_1$ is a constant approximation of the optimal solution with high probability.
	On the other hand, for any $u \in L$, $\Pr[u \in T_2]  = k / |T_1| \geq k / r$.
	By Lemma \ref{Lemma: property of submodular functions}, $\E[f(T_2)] \geq (k/r)f(T_1)$.
	Since $r = O(k)$, $k/r$ is a constant.
	Thus $T_2$ is still a constant approximation of the optimal solution with high probability.
\end{proof}
	
\paragraph{Assumption \ref{assump:dist}.2 Is Necessary.}

\begin{theorem}
	There is a distribution $\mathcal{D}$, which satisfies Assumption \ref{assump:dist}.1 and \ref{assump:dist}.3, but not Assumption \ref{assump:dist}.2,
	such that coverage functions are not $\alpha$-optimizable under OPSS for any constant $\alpha$ in the cardinality constraint $\mathcal{M}_{\leq k}$ over $\mathcal{D}$.
\end{theorem}

\begin{proof}
	The distribution $\mathcal{D}$ is constructed as follows.
	Let $L$ be partitioned into two disjoint subsets $L_1$ and $L_2$; each contains exactly $n / 2$ nodes.
	Let $\mathcal{D}$ be the uniform distribution over all subsets of $L_2$ with size exactly $k$.
	Clearly, this distribution satisfies Assumption \ref{assump:dist}.1 and \ref{assump:dist}.3, but not Assumption \ref{assump:dist}.2.
	
	Let $|L| = n$, $|R| = m$ and $L_1 = \{u_1, u_2, \cdots, u_{n/2}\}$.
	We first construct a class of graphs $G_1, \cdots, G_{n/2}$ where $G_i = (L, R, E_i)$.
	For $G_i$, $N_{G_i}(u_i) = R$ and for $u \neq u_i$, $N_{G_i}(u) = \emptyset$.
	Clearly, the optimal solution of any graph covers $m$ nodes.
	For any $G_i$ and sample $S \sim \mathcal{D}$, $N_{G_i}(S_i) = \emptyset$.
	Thus those graphs cannot be distinguished from the samples.
	
	We prove the desired ratio by a probabilistic argument.
	Let $G$ be a graph drawn uniformly at random from $G_1, \cdots, G_{n/2}$.
	Let $B$ be any (randomized) algorithm and $T$ be the solution it returns.
	Suppose that the solution $T$ of $B$ is fixed.
	Since $|T| \leq k$, $\Pr_G[N_G(T) = R \mid \mbox{the solution $T$ of } B \mbox{ is fixed}] \leq k / (n/2) = o(1)$.
	Since $B$ cannot distinguish those graphs from the samples, the solution $T$ of $B$ is independent of the random graph $G$.
	As a result, $\Pr_{B, G}[N_G(T) = R] = o(1)$ and $\E_{B, G}[f_G(T)] = o(1) \cdot m$, which implies there must be a $G$ from $G_1, \cdots, G_p$ such that $\E_{B}[f_G(T)] = o(1) \cdot m$.
	Thus $\E_{B}[f_G(T)] / f_G(OPT) = o(1)$.
\end{proof}
	
}

%%%%%%%%%%%%%%%%%%%%%%%%%%%%%%%%%%%%%%%%%%%%%%%%%%%%%%%%%%%%%%%%%%%%%%%%%%%%%%%
%%%%%%%%%%%%%%%%%%%%%%%%%%%%%%%%%%%%%%%%%%%%%%%%%%%%%%%%%%%%%%%%%%%%%%%%%%%%%%%
% DELETE THIS PART. DO NOT PLACE CONTENT AFTER THE REFERENCES!
%%%%%%%%%%%%%%%%%%%%%%%%%%%%%%%%%%%%%%%%%%%%%%%%%%%%%%%%%%%%%%%%%%%%%%%%%%%%%%%
%%%%%%%%%%%%%%%%%%%%%%%%%%%%%%%%%%%%%%%%%%%%%%%%%%%%%%%%%%%%%%%%%%%%%%%%%%%%%%%
%\appendix
%\section{Do \emph{not} have an appendix here}
%
%\textbf{\emph{Do not put content after the references.}}
%%
%Put anything that you might normally include after the references in a separate
%supplementary file.
%
%We recommend that you build supplementary material in a separate document.
%If you must create one PDF and cut it up, please be careful to use a tool that
%doesn't alter the margins, and that doesn't aggressively rewrite the PDF file.
%pdftk usually works fine.
%
%\textbf{Please do not use Apple's preview to cut off supplementary material.} In
%previous years it has altered margins, and created headaches at the camera-ready
%stage.
%%%%%%%%%%%%%%%%%%%%%%%%%%%%%%%%%%%%%%%%%%%%%%%%%%%%%%%%%%%%%%%%%%%%%%%%%%%%%%%
%%%%%%%%%%%%%%%%%%%%%%%%%%%%%%%%%%%%%%%%%%%%%%%%%%%%%%%%%%%%%%%%%%%%%%%%%%%%%%%

\end{document}